\documentclass{article}

\usepackage{arxiv}

\usepackage[utf8]{inputenc} 
\usepackage[T1]{fontenc}    
\usepackage{hyperref}       
\usepackage{url}            
\usepackage{booktabs}       
\usepackage{amsfonts}       
\usepackage{nicefrac}       
\usepackage{microtype}      
\usepackage{lipsum}		
\usepackage[numbers,sort&compress]{natbib}
\usepackage{doi}

\usepackage{hyperref}
\usepackage{url}
\usepackage{graphicx}
\usepackage{amsmath}
\usepackage{amssymb}
\usepackage{amsthm}
\usepackage{booktabs}
\usepackage{multirow}
\usepackage[table]{xcolor}
\usepackage{adjustbox}
\usepackage{tabularx}
\usepackage{pifont}
\usepackage{tcolorbox}

\definecolor{CRCgreen}{RGB}{222,242,222}
\definecolor{Riskred}{RGB}{250,224,224}
\definecolor{Neargray}{RGB}{235,235,235}

\newcommand{\riskvalue}[2]{%
  #1\,{\scriptsize$\pm$#2}%
}

\newcommand{\rcok}[2]{%
  \cellcolor{CRCgreen}%
  \riskvalue{#1}{#2}\,\ding{51}%
}

\newcommand{\rcnear}[2]{%
  \cellcolor{Neargray}%
  \riskvalue{#1}{#2}$^\dagger$%
}

\newcommand{\rcbad}[2]{%
  \cellcolor{Riskred}%
  \riskvalue{#1}{#2}\,\ding{55}%
}

\newcommand{\rcplain}[2]{%
  \riskvalue{#1}{#2}%
}

\newcommand{\rcabstain}{%
  \textit{abst.}$^\ddagger$%
}

\definecolor{STLBlue}{RGB}{56,126,184}
\definecolor{SpaPurple}{RGB}{135,100,175}

\newtheorem{theorem}{Theorem}

\title{\Large\normalfont \textbf{SCP-NL2TL: Selective Conformal Prediction with Semantic Verification for Natural Language to Temporal Logic Specifications
}}

\date{} 					

\author{
Yixuan Wang$^{1}$,
Licheng Luo$^{1}$,
Yu Fu$^{1}$,
Kaidi Xu$^{2}$,
Yue Dong$^{1}$,
Mingyu Cai$^{1}$
\\[0.6em]
$^{1}$University of California, Riverside,
Riverside, CA, USA
\\
$^{2}$City University of Hong Kong,
Hong Kong SAR, China
\\[0.4em]
\texttt{
ywang1457@ucr.edu,
lichengl@ucr.edu,
yfu093@ucr.edu
}
\\
\texttt{
kaidixu at cityu.edu.hk,
yue.dong@ucr.edu,
mingyuc@ucr.edu
}
}

\renewcommand{\shorttitle}{\textit{arXiv} Template}

\hypersetup{
pdftitle={A template for the arxiv style},
pdfsubject={q-bio.NC, q-bio.QM},
pdfauthor={David S.~Hippocampus, Elias D.~Striatum},
pdfkeywords={First keyword, Second keyword, More},
}

\begin{document}
\maketitle

\begin{abstract}
Translating natural language instructions into machine-interpretable formal specifications enables robots and autonomous systems to plan, reason, and formally verify their behavior. However, existing translation models typically generate a specification for every input, even when the result is unreliable or fails to capture the user's intent, creating risks in safety-critical applications. Inspired by selective conformal prediction, we propose a selective translation framework that not only generates formal specifications but also determines when they can be trusted. Reliability is scored by two complementary black-box signals, the fidelity of the specification back-translated into natural language and the dispersion of repeated translations under exact semantic equivalence, which fail on different errors and jointly separate incorrect translations more sharply than either alone. Conformal risk control calibrates this score into a decision that accepts a specification or abstains, with a distribution-free bound on the rate at which incorrect specifications are accepted for execution, and a conformal anomaly detector on instruction embeddings screens out-of-distribution inputs before any translation is attempted. The proposed framework is general across formal specification languages, with experiments on Signal Temporal Logic (STL), Linear Temporal Logic (LTL), and geometric Spatio-Temporal Logic (SpaTiaL) demonstrating improved translation reliability, robustness under the evaluated cross-tier shifts, and effective uncertainty-aware abstention. This work establishes a foundation for trustworthy natural language interfaces by enabling AI systems to recognize when generated specifications may not be reliable. Project materials are available at
\href{https://sites.google.com/ucr.edu/scpnl2tl?usp=sharing}
{\texttt{sites.google.com/ucr.edu/scpnl2tl}}.
\end{abstract}


\section{Introduction}
Translating natural language instructions into machine-interpretable formal specifications enables robots and autonomous systems to plan, reason, and formally verify their behavior~\cite{xu2021adaptive, roy2023learning,chakraborti2024interactive,zhao2019probabilistic}. Formal languages such as Linear Temporal Logic (LTL), Signal Temporal Logic (STL), and geometric Spatio-Temporal Logic (SpaTiaL) provide representations of temporal and spatial requirements ~\cite{pnueli1977temporal,maler2004monitoring,pek2023spatial}, but writing them requires formal methods expertise. Recent natural language to temporal logic (NL2TL) systems have made this process more accessible using semantic parsing, sequence models, and, more recently, pretrained language models with structured intermediate representations~\cite{he2022deepstl,chen2023nl2tl,luo2025nl2spatial}. Although these methods substantially improve translation accuracy and structural validity, they invariably generate a specification for every input.

Always producing a specification is undesirable in safety-critical applications. A formula may be syntactically valid yet semantically incorrect by assigning an incorrect predicate, argument, temporal operator, interval, logical scope, or spatial relation, causing downstream planners or verifiers to execute behavior that does not reflect the user's intent. A practical NL2TL system should therefore determine not only what specification to generate, but also whether it is sufficiently reliable to be executed. Standard evaluation metrics offer no such judgment, as exact-match accuracy characterizes average benchmark performance rather than the trustworthiness of an individual translation~\cite{stengel2023calibrated, chen2023unified,zhong2020semantic,somov2025confidence}.
Conformal prediction (CP) supplies the statistical machinery for such a judgment. From any heuristic score and a set of calibration examples, it produces a prediction set that covers the true label at a user-specified rate, requiring only that calibration and test data be exchangeable ~\cite{vovk2005algorithmic}. The guarantee, however, is stated over sets, and a planner executes one formula rather than weighing several~\cite{geifman2019selectivenet, lee2024selective}; a set of candidate specifications has nothing to hand it. Recent work has redirected this machinery from covering an answer toward declining to give one. Conformal risk control generalizes the coverage target to any bounded monotone loss ~\cite{ICLR2024_f3549ef9}, and selective conformal procedures add a second decision ahead of the first, testing whether the calibration data speak for a given test point and withholding a response when they do not ~\cite{wang-etal-2025-sconu}. Across this line, abstention becomes an outcome the procedure is calibrated to produce rather than an admission that it failed.

We propose a selective layer to any translator that either returns or withholds its generated formula, calibrating the decision so the rate of incorrect specifications reaching execution satisfies a user-defined risk budget. Unlike prior conformal methods that calibrate prediction sets using model confidence, our framework treats the translator as a black box and directly controls the risk of the single specification executed by the planner. Instead of relying on token likelihoods, we score translations using two complementary semantic signals: back-translation, which measures agreement with the instruction, and repeated sampling, which measures semantic consistency across translations. An instruction-level conformal test further rejects out-of-distribution inputs before translation. Across STL, LTL, and SpaTiaL, our method consistently satisfies the target risk budget, while coverage-calibrated baselines do not, and the instruction-level filter further improves robustness under distribution shift.
The core contributions can be summarized as follows:

\begin{itemize}
    \item To the best of our knowledge, this is the first selective translation framework for NL2TL that augments any translator with an accept-or-abstain mechanism and bounds the rate at which incorrect temporal logic specifications are accepted for execution.
    \item We introduce a semantic consistency score that combines back-translation fidelity and semantic agreement across repeated translations, improving discrimination between correct and incorrect specifications.
    \item We incorporate an instruction-level conformal anomaly detector before calibration to identify out-of-distribution inputs, bound in-distribution deferrals, and improve robustness under distribution shift.
\end{itemize}

\noindent\textbf{Related Work.}
Translating natural language into formal task specifications has been studied using grammar-based methods, semantic parsing, and sequence models for robotics and cyber-physical systems~\cite{raman2013sorry,gopalan2018sequence,wang2021learning,he2022deepstl}. Recent approaches leverage pretrained language models, structured intermediate representations, retrieval augmentation, model-checker feedback, and grammar-constrained decoding to improve translation accuracy and structural validity~\cite{chen2023nl2tl,cosler2023nl2spec,mendoza2024translating,fang2025enhancing,english2025grammar,luo2025nl2spatial,fuggitti2023nl2ltl}. These methods focus on generating or refining executable specifications, but they always return a formula regardless of its reliability. ConformalNL2LTL attains a guaranteed translation success rate by building the formula through conformal QA steps that escalate uncertain decisions to a stronger model or the user~\cite{sundarsingh2025conformalnl2ltl}. Our work is complementary: instead of modifying the translator, we develop a model-agnostic selection layer that determines whether a generated specification should be accepted or withheld.

Conformal prediction has recently been applied to language models for generation sets, prompt selection, API-only uncertainty estimation, response validity, factuality, and selective generation~\cite{su2024api,wang2024conu,quach2024conformal,mohri2024language,zollo2024prompt,cherian2024large}. Related work has also studied semantic uncertainty~\cite{farquhar2024detecting}, conformal procedures under non-exchangeability and distribution shift~\cite{farinhas2024non,tibshirani2019conformal,barber2023conformal},  and conformal significance testing for uncertainty outlier detection~\cite{wang-etal-2025-sconu}. A parallel line acquires specifications from system behavior rather than language, learning STL formulas through differentiable robustness and neural structures~\cite{leung2023backpropagation,li2024tlinet,li2023learning,luo2026differentiable}, with conformal prediction certifying the inferred formulas~\cite{li2025conformal,wang2026conformalized}. Conformal risk control further extends conformal calibration from coverage guarantees to bounded monotone losses~\cite{ICLR2024_f3549ef9}. The language-model frameworks calibrate over a prediction set and derive their nonconformity from model confidence, token likelihoods, or sampled responses, none of which is available for a black-box translator whose output is a single executable formula. Conformal prediction has also entered language-instructed robot planning, calibrating when an LLM planner should proceed, seek help, or escalate~\cite{wang2023conformal,wang2024probabilistically,ren2023robots,lindemann2023safe}.

Selective conformal procedures add a decision ahead of the conformal one. SConU tests each incoming sample against the uncertainty distribution of the calibration set and declines to answer when the evidence for exchangeability is insufficient ~\cite{wang-etal-2025-sconu}, building on conformal outlier testing ~\cite{bates2023testing,jin2023selection} and on rank-based anomaly detection under exchangeability ~\cite{laxhammar2015inductive}. A related line pursues conditional rather than marginal risk guarantees through two-stage calibration, at the cost of a modified exchangeability requirement ~\cite{xu2025selective}. Nearest neighbours in embedding space serve as an out-of-distribution statistic in the detection literature independently of conformal calibration ~\cite{pmlr-v162-sun22d}. We place such a statistic inside the admission gate, computed from the instruction alone so that screening precedes translation, and calibrate what follows it on the rate at which incorrect specifications are accepted.

\section{Background and Problem}
\label{sec: preliminaries}

\textbf{Formal Logic Languages.}
Temporal Logic (TL) specifications as formal languages provide mathematically precise descriptions of desired system behaviors, enabling automated planning and formal verification for autonomous and cyber-physical systems ~\cite{pnueli1977temporal,baier2008principles,belta2017formal}. Unlike natural language, which is inherently ambiguous, a formal specification has well-defined syntax and semantics, allowing its correctness to be rigorously analyzed and verified.
Let $\Phi$ denote the space of well-formed formulas in a target formal language. Although numerous formal specification languages have been proposed, including LTL~\cite{pnueli1977temporal}, STL~\cite{maler2004monitoring}, and  SpaTiaL~\cite{pek2023spatial}, they share a common compositional structure. A specification is recursively constructed from \emph{atomic predicates}, \emph{Boolean operators}, and \emph{domain-specific operators}, such as temporal or spatial modalities:
\begin{equation}
\varphi
::=
\mu
\mid
\neg\varphi
\mid
\varphi_1\land\varphi_2
\mid
\varphi_1\lor\varphi_2
\mid
\mathcal{O}(\varphi),
\label{eq:general_grammar}
\end{equation}
where $\mu$ is an atomic predicate and $\mathcal{O}$ denotes one or more logic-specific operators. For example, LTL employs temporal operators such as \textit{Eventually} ($\mathbf{F}$), \textit{Always} ($\mathbf{G}$), and \textit{Until} ($\mathbf{U}$); STL augments these operators with bounded time intervals over continuous-valued signals; and SpaTiaL further introduces geometric predicates and spatial operators. Each logic pairs this grammar with a formal semantics specifying when a system trajectory satisfies a specification, which is what makes generated formulas executable and verifiable downstream.



\noindent\textbf{Natural Language to Formal Specification.}

Given a natural language instruction $u\in\mathcal{U}$, NL2TL aims to generate a formal specification $\hat{\varphi}=f_\theta(u)\in\Phi$. Most approaches learn the conditional distribution $p_\theta(\varphi|u)$ and predict the most likely specification~\cite{chen2023nl2tl,cosler2023nl2spec,mendoza2024translating,luo2025nl2spatial}, but the predicted formula may be syntactically valid while failing to preserve the semantic intent of the instruction, and likelihood under $p_\theta$ offers no measure of this semantic correctness~\cite{duan2024shifting}.

\noindent\textbf{Conformal Prediction (CP).}
CP provides a model-agnostic framework for quantifying prediction reliability with finite-sample, distribution-free guarantees~\cite{vovk2005algorithmic}. Given a calibration set and a nonconformity score $s(x,y)$, CP computes a threshold from calibration scores and constructs a prediction set
$\mathcal{C}_\alpha(x)=\{y:s(x,y)\le\hat q\},$
which satisfies
$\Pr\!\left(Y\in\mathcal{C}_\alpha(X)\right)\ge 1-\alpha$
under exchangeability. While this guarantee ensures the correct label is contained in the prediction set with high probability, it is not directly suitable for NL2TL, where downstream planners typically require a single executable formal specification rather than a set of candidates. This limitation motivates conformal methods tailored to selective and reliable formal specification generation.


\noindent\textbf{Problem Formulation.}
Let $x\in\mathcal{X}$ be a natural-language instruction and let $\Phi$ be the space of well-formed formulas in a target specification language, as in the preliminaries; the formulation applies to any logic generated by the grammar in Eq.~\ref{eq:general_grammar}, with STL, LTL, and SpaTiaL serving as our experimental instances. A translator $\hat{\varphi} = f(x) \in \Phi$ is treated as a black box: the framework uses only the final pair $(x,\hat{\varphi})$, with no access to parameters, likelihoods, or internal states. The translator may be of any construction~\cite{luo2025nl2spatial,liu2022lang2ltl,chen2023nl2tl}; only the instruction and the final formula enter the framework. We augment $f$ with a selection function $ g(x,\hat{\varphi}) \in \{0,1\}$, where $g=1$ indicates that $\hat{\varphi}$ is returned and $g=0$ indicates abstention. Let $\varphi^\star$ denote the reference formula and define the translation error as $z(x,\hat{\varphi},\varphi^\star)
    = \mathbf{1}\!\left[\hat{\varphi}\not\equiv\varphi^\star\right]$,
where $\equiv$ is the semantic-equivalence criterion of the target
benchmark. The selective translator is evaluated by its joint risk
and coverage,
\begin{equation}
    \mathcal{R}_{\mathrm{joint}}
    = \mathbf{E}\!\left[\, g(x,\hat{\varphi})\cdot
      z(x,\hat{\varphi},\varphi^\star) \,\right],
    \qquad
    \mathcal{C}
    = \mathbf{P}\!\left(g(x,\hat{\varphi})=1\right).
    \label{eq:joint_risk}
\end{equation}
The joint risk is the probability that a translation is both accepted and incorrect, taken over the full input stream rather than over the accepted subset; an abstention contributes zero regardless of the underlying error.

\noindent\textbf{Challenges and Motivation.}
Designing a conformal framework for NL2TL differs fundamentally from standard classification. First, the output space $\Phi$ is effectively unbounded, making prediction sets impractical while downstream planners require a single executable specification or abstention. Second, the objective is selective risk rather than coverage: standard conformal prediction guarantees that the correct specification belongs to a prediction set but does not control the probability that an accepted translation is incorrect, nor the joint rate at which incorrect translations reach execution. Third, conventional nonconformity scores depend on model confidences or token likelihoods, which are unavailable for black-box translators and often poorly aligned with semantic correctness. Therefore, reliable NL2TL requires a framework that (i) performs accept-or-abstain decisions instead of prediction-set construction, (ii) calibrates acceptance risk rather than set coverage, and (iii) uses scores computable solely from the instruction and the generated specification.

\section{Methodology}
\label{sec:method}
\begin{figure}[t]
\centering
\includegraphics[width=0.70\columnwidth]{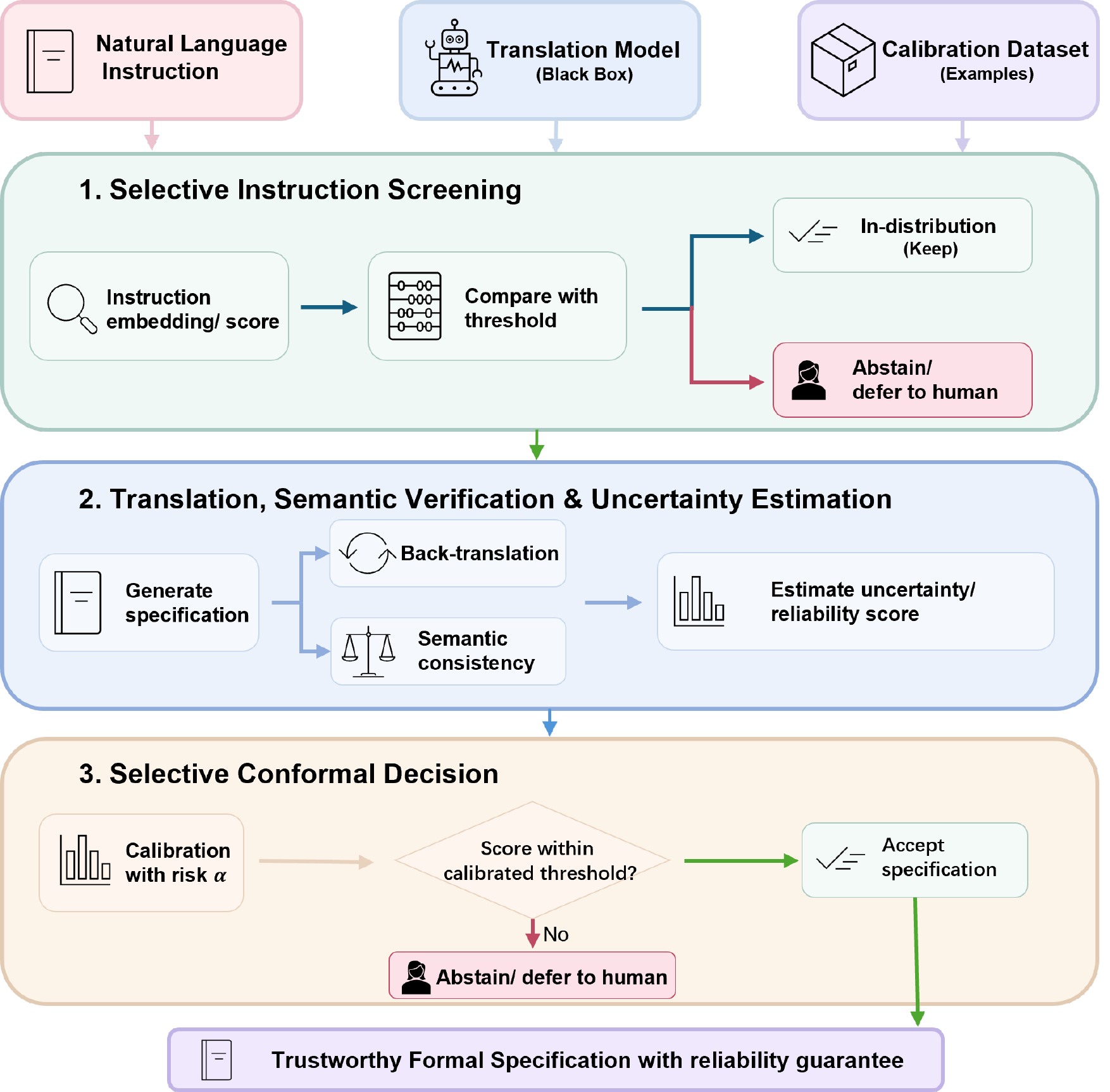}
\caption{Overview of the selective translation framework: instructions are screened before translation (level~$\delta$), scored by two complementary signals after it, and returned only when a threshold calibrated to the joint-risk budget~$\alpha$ admits the evidence.}
\label{fig:framework}
\vspace{-10pt}
\end{figure}

Conformal risk control (CRC)~\cite{ICLR2024_f3549ef9} extends conformal calibration from coverage guarantees to arbitrary bounded monotone losses by selecting the largest acceptance threshold whose expected loss remains below a prescribed budget under exchangeability. However, CRC provides only the calibration mechanism, not the scoring function, selection criterion, or robustness to distribution shift. Building on CRC, we propose a black-box conformal selection framework that augments any NL2TL translator with an accept-or-abstain decision layer: an instruction-level screen, two semantic scores, and a calibrated threshold (Fig.~\ref{fig:framework}). 
The objective is to maximize acceptance subject to a prescribed budget $\alpha\in(0,1)$ on the joint risk of Eq.~\ref{eq:joint_risk}:
\begin{equation}
\max_{g}\;\mathcal{C}
\qquad\text{s.t.}\qquad
\mathcal{R}_{\mathrm{joint}}\le\alpha .
\label{eq:objective}
\end{equation}
\subsection{Semantic Consistency Scoring}
\label{sec:score}

The decision function in Eq.~\ref{eq:objective} requires a scalar nonconformity score computed from the instruction--translation pair $(x,\hat\varphi)$ that assigns higher values to less reliable translations. Our default score is based on back-translation. A back-translator $B$ (an LLM prompted with the target logic) converts $\hat\varphi$ into natural language, $\hat u=B(\hat\varphi)$, and an LLM judge evaluates its semantic consistency with the original instruction across four dimensions: logical structure, temporal operators, time constraints, and overall meaning. The average agreement $A(x,\hat u)\in[0,1]$ defines the nonconformity score
\begin{equation}
S_{\mathrm{bt}}(x,\hat\varphi)
= 1-A\left(x,B(\hat\varphi)\right).
\label{eq:sbt}
\end{equation}
Invalid or unparsable formulas receive the maximum score, $S_{\mathrm{bt}}=1$, enforcing syntactic correctness before semantic evaluation. This design isolates logic-specific knowledge within the back-translator, while the judge operates entirely in natural language, making the scoring pipeline applicable to any target logic generated by Eq.~\ref{eq:general_grammar}. The approach combines the round-trip principle from machine translation~\cite{moon2020revisiting} with the LLM-as-judge paradigm~\cite{zheng2023judging}.
Because back-translation inspects only the final output, a complementary score measures the stability of repeated stochastic translations. Given $k$ samples, translator uncertainty is measured by
\begin{equation}
    S_{\mathrm{sc}}(x)
    = 1 - \frac{1}{k}\,
      \max_{1\le j\le k}\;
      \Bigl|\bigl\{\, m :
      \hat\varphi^{(m)} \simeq \hat\varphi^{(j)} \,\bigr\}\Bigr|,
    \label{eq:ssc}
\end{equation}
where the maximum is the size of the largest cluster of mutually equivalent samples, $\simeq$ denotes semantic equivalence under logic-specific canonicalization, and invalid formulas are equivalent only to themselves. This extends self-consistency~\cite{wang2023selfconsistency} to formal specifications, where equivalence can be checked exactly~\cite{farquhar2024detecting}. Since $S_{\mathrm{bt}}$ detects semantic mismatches while $S_{\mathrm{sc}}$ captures translation uncertainty, their equal-weight fusion
\begin{equation}
S_{\mathrm{fu}}
= \tfrac{1}{2}\left(S_{\mathrm{bt}}+S_{\mathrm{sc}}\right)
\label{eq:fusion}
\end{equation}
provides the default reliability score. When repeated sampling is unavailable, such as for deterministic or offline translators, $S_{\mathrm{bt}}$ alone can be used without changing the subsequent calibration procedure, which accepts any score in $[0,1]$.

\subsection{Risk-Calibrated Selection}
\label{sec:verification}

The score orders translations by suspicion but does not say where acceptance should stop; that boundary must come from data. A threshold is calibrated on translations whose outcomes are known,
\begin{align}
    \mathcal{D}_{\mathrm{cal}}
    &= \{(x_i,\hat\varphi_i,\varphi_i^\star)\}_{i=1}^{n},
    \label{eq:calset}\\
    S_i &= S(x_i,\hat\varphi_i),
    \qquad
    z_i = \mathbf{1}\!\left[\hat\varphi_i\not\equiv\varphi_i^\star\right],
    \label{eq:calquant}
\end{align}
where $S$ is the deployed consistency score, either the fusion $S_{\mathrm{fu}}$ or the output-only $S_{\mathrm{bt}}$. A candidate threshold $\tau$ accepts every translation scoring at or below it, and its empirical joint risk on the calibration set is
\begin{equation}
    \widehat{\mathcal{R}}(\tau)
    = \frac{1}{n}\sum_{i=1}^{n}
      \mathbf{1}\!\left[S_i\le\tau\right] z_i .
    \label{eq:empirical_joint}
\end{equation}

Every calibration sample contributes to this sum: an abstention counts zero, an accepted error counts one. The quantity is therefore an average over the full sample rather than over the accepted subset, which is what permits the direct application of conformal risk control. Following the calibration rule of ~\citet{ICLR2024_f3549ef9}, the threshold is
\begin{equation}
    \hat\tau
    = \max\Bigl\{\tau\in\mathcal{T} :
      \tfrac{n}{n+1}\,\widehat{\mathcal{R}}(\tau)
      + \tfrac{1}{n+1} \le \alpha \Bigr\},
    \label{eq:crc_threshold}
\end{equation}
where $\mathcal{T}$ is the set of distinct calibration scores together with $-\infty$, and ties are resolved by admitting all samples that share a score. The correction term accounts for the unseen test point and is what turns an empirical constraint into a population guarantee. At test time,
\begin{equation}
    g(x,\hat\varphi)
    = \mathbf{1}\!\left[S(x,\hat\varphi)\le\hat\tau\right].
    \label{eq:decision}
\end{equation}

\begin{theorem}[Joint risk control]
\label{thm:crc}
Fix a calibration group. If the calibration pairs of this group and the test pair are exchangeable and $\hat\tau$ is chosen by Eq.~\ref{eq:crc_threshold} on this group's calibration scores, then the decision rule of Eq.~\ref{eq:decision} satisfies, for test
inputs of the same group,
\[
\mathcal{R}_{\mathrm{joint}}
= \mathbf{E}\!\left[\,g(X,\hat\varphi)\cdot Z\,\right]
\le \alpha .
\]
\end{theorem}
\begin{proof}
The loss $\ell_\tau = \mathbf{1}[S\le\tau]\,z$ is bounded in $[0,1]$ and non-increasing as $\tau$ decreases, since lowering the threshold can only remove accepted samples. The claim is the conformal risk control theorem of ~\citet{ICLR2024_f3549ef9} applied to this loss.
\end{proof}
The guarantee in Theorem~\ref{thm:crc} holds in expectation over the calibration and test samples. Consequently, a threshold that fully utilizes the risk budget will achieve expected risk close to $\alpha$, while individual realizations may exceed $\alpha$.
Feasibility depends on both the calibration size and score distribution. Since the empirical risk is evaluated at resolution $\tfrac{1}{n+1}$, risk budgets below this level cannot be certified, analogous to the minimum conformal $p$-value attainable with $n$ calibration samples~\cite{wang-etal-2025-sconu,vovk2005algorithmic}. Moreover, tied scores force samples to enter the acceptance region together. If the smallest nonempty acceptance region contains $L_0$ calibration errors, a feasible threshold exists only when
\begin{equation}
\alpha \ge \alpha_\ell=\frac{L_0+1}{n+1}.
\label{eq:alpha_l}
\end{equation}
Otherwise, Eq.~\ref{eq:crc_threshold} returns $\hat\tau=-\infty$, causing the system to abstain on every input rather than provide an uncertifiable guarantee. Since $\alpha_\ell$ is determined entirely by the calibration data, feasibility can be assessed before deployment. It can be improved by increasing the calibration set or relaxing the risk budget. Calibrating each group separately in this way yields group-conditional guarantees, at the cost of fewer calibration samples per group and thus a potentially larger $\alpha_\ell$.

Coverage is maximized by construction through the largest feasible threshold in Eq.~\ref{eq:crc_threshold}. Its performance is naturally bounded by two references. For a translator with error rate $p$ under budget $\alpha$, the maximum achievable acceptance rate is $\min(1,\,1-p+\alpha)$, attained by a perfect score that ranks all correct translations above all incorrect ones. At the other extreme, an uninformative score achieves acceptance rate $\alpha/p$ by accepting a random subset with the population error rate. The realized acceptance rate therefore quantifies how closely the score approaches the optimal ranking.

\subsection{Filtering Out-of-Distribution Instructions}
\label{sec:filtering}

Theorem~\ref{thm:crc} assumes exchangeability between calibration and test data, an assumption that may fail under distribution shift. We therefore introduce an upstream input filter that defers instructions outside the calibration distribution before translation and scoring. Unlike the CRC threshold in Eq.~\ref{eq:crc_threshold}, which abstains when translation evidence is insufficient, the filter abstains when no statistical guarantee can be provided.

The filter measures the atypicality of an instruction using its embedding. Scores are computed against a fixed reference set $\mathcal{X}_{\mathrm{ref}}$ of training-sourced instructions, disjoint from the instructions on which the deferral level is calibrated and from the test stream. Let $e(x)\in\mathbb{R}^d$ denote the normalized embedding of instruction $x$, and define its average $k$-nearest-neighbor distance to the reference set as
\begin{equation}
D(x)=\frac{1}{k}\sum_{x_j\in\mathcal{N}_k(x)}
\bigl\|e(x)-e(x_j)\bigr\|_2 ,
\label{eq:knn}
\end{equation}
where $\mathcal{N}_k(x)$ is the set of $k$ nearest reference embeddings. The same function scores the $m$ deferral-calibration instructions $x_1,\ldots,x_m$ and any test instruction, and the conformal $p$-value of a test instruction is
\begin{equation}
p(x)
=
\frac{1+\sum_{i=1}^{m}\mathbf{1}\!\left[D(x_i)\ge D(x)\right]}
{m+1},
\label{eq:pvalue}
\end{equation}


\begin{theorem}[Filter validity]
\label{thm:filter}
If the test instruction is exchangeable with the deferral-calibration instructions, then the filter defers it with probability at most $\delta$,
\[
\mathbf{P}\!\left(p(X)<\delta\right)\le\delta .
\]
\end{theorem}
\begin{proof}
The score function $D$ is fixed by $\mathcal{X}_{\mathrm{ref}}$ and depends on none of $x_1,\ldots,x_m$ or the test instruction, so exchangeability of the instructions carries over to their scores. The rank of $D(X)$ among $\{D(x_1),\ldots,D(x_m),D(X)\}$ is then
uniformly distributed up to ties, and ties only enlarge $p(X)$ through the inequality in Eq.~\ref{eq:pvalue}, so $\mathbf{P}(p(X)<\delta)\le\delta$.
\end{proof}

Theorem~\ref{thm:filter} guarantees that an in-distribution instruction is deferred with probability at most $\delta$. Under distribution shift, the filter removes detected atypical inputs before translation; Theorem~\ref{thm:crc} applies to the retained stream to the extent that its exchangeability with the calibration
data is preserved. Since no finite-sample procedure can detect every shifted input, the filter replaces the assumption of no distribution shift with the weaker assumption that significant shifts are detected, and the residual risk under partial detection is evaluated experimentally.

\section{Experimental Results}

\begin{table*}[t]
\centering
\caption{
Error-detection AUROC by scoring channel, grouped by information source.
Best values in each column are shown in bold. Darker cells indicate higher
AUROC within each formalism. The $\Delta$ row reports the fusion margin
over the best single channel.
}
\label{tab:scoring_ablation}

\scriptsize
\setlength{\tabcolsep}{2.6pt}
\renewcommand{\arraystretch}{1.14}

\begin{adjustbox}{max width=\textwidth}
\begin{tabular}{@{}llcccccccc@{}}
\toprule

& &
\multicolumn{4}{c}{
    \cellcolor{STLBlue!10}
    \textbf{STL (GPT-5.2 few-shot)}
}
&
\multicolumn{4}{c}{
    \cellcolor{SpaPurple!10}
    \textbf{SpaTiaL (fine-tuned LLaMA)}
}
\\

\cmidrule(lr){3-6}
\cmidrule(lr){7-10}

\textbf{Source}
&
\textbf{Channel}
&
\textbf{D2}
&
\textbf{D3}
&
\textbf{D4}
&
\textbf{All}
&
\textbf{D2$^\dagger$}
&
\textbf{D3}
&
\textbf{D4}
&
\textbf{All}
\\

\midrule


Intrinsic
&
self-consistency $S_{\mathrm{sc}}$
&
\cellcolor{STLBlue!12}0.614
&
\cellcolor{STLBlue!16}0.644
&
\cellcolor{STLBlue!24}0.693
&
\cellcolor{STLBlue!20}0.655
&
\cellcolor{SpaPurple!16}0.773
&
\cellcolor{SpaPurple!16}0.786
&
\cellcolor{SpaPurple!28}0.906
&
\cellcolor{SpaPurple!24}0.867
\\

\midrule


\multirow{4}{*}{Extrinsic}
&
back-translation mean-4 $S_{\mathrm{bt}}$
&
\cellcolor{STLBlue!8}0.586
&
\cellcolor{STLBlue!8}0.580
&
\cellcolor{STLBlue!28}0.720
&
\cellcolor{STLBlue!16}0.634
&
\cellcolor{SpaPurple!16}0.774
&
\cellcolor{SpaPurple!20}0.815
&
\cellcolor{SpaPurple!28}0.918
&
\cellcolor{SpaPurple!24}0.871
\\

&
direct judge mean-4
&
\cellcolor{STLBlue!12}0.614
&
\cellcolor{STLBlue!4}0.520
&
\cellcolor{STLBlue!28}0.728
&
\cellcolor{STLBlue!16}0.630
&
\cellcolor{SpaPurple!8}0.661
&
\cellcolor{SpaPurple!24}0.858
&
\cellcolor{SpaPurple!28}0.939
&
\cellcolor{SpaPurple!24}0.892
\\

&
single-question
&
\cellcolor{STLBlue!16}0.640
&
\cellcolor{STLBlue!8}0.569
&
\cellcolor{STLBlue!24}0.700
&
\cellcolor{STLBlue!16}0.635
&
\cellcolor{SpaPurple!24}0.870
&
\cellcolor{SpaPurple!20}0.824
&
\cellcolor{SpaPurple!24}0.881
&
\cellcolor{SpaPurple!24}0.863
\\

&
embedding cosine $S_{\mathrm{emb}}$
&
\cellcolor{STLBlue!20}0.673
&
\cellcolor{STLBlue!4}0.536
&
\cellcolor{STLBlue!8}0.571
&
\cellcolor{STLBlue!8}0.585
&
\cellcolor{SpaPurple!20}0.825
&
\cellcolor{SpaPurple!4}0.599
&
\cellcolor{SpaPurple!8}0.691
&
\cellcolor{SpaPurple!8}0.662
\\

\midrule


\multirow{2}{*}{Fusion}
&
3-judge mean
&
\cellcolor{STLBlue!20}0.660
&
\cellcolor{STLBlue!8}0.556
&
\cellcolor{STLBlue!28}0.727
&
\cellcolor{STLBlue!20}0.656
&
\cellcolor{SpaPurple!20}0.842
&
\cellcolor{SpaPurple!24}\textbf{0.875}
&
\cellcolor{SpaPurple!28}0.944
&
\cellcolor{SpaPurple!28}0.915
\\

&
$\frac{1}{2}\left(S_{\mathrm{sc}}+S_{\mathrm{bt}}\right)$
&
\cellcolor{STLBlue!20}\textbf{0.679}
&
\cellcolor{STLBlue!20}\textbf{0.663}
&
\cellcolor{STLBlue!28}\textbf{0.752}
&
\cellcolor{STLBlue!24}\textbf{0.706}
&
\cellcolor{SpaPurple!20}0.824
&
\cellcolor{SpaPurple!24}0.870
&
\cellcolor{SpaPurple!28}\textbf{0.953}
&
\cellcolor{SpaPurple!28}\textbf{0.920}
\\

\cmidrule(lr){2-10}

&
$\Delta$ over best single
&
\cellcolor{STLBlue!4}$+0.006$
&
\cellcolor{STLBlue!4}$+0.018$
&
\cellcolor{STLBlue!4}$+0.025$
&
\cellcolor{STLBlue!4}$+0.051$
&
\cellcolor{SpaPurple!4}--
&
\cellcolor{SpaPurple!4}$+0.012$
&
\cellcolor{SpaPurple!4}$+0.015$
&
\cellcolor{SpaPurple!4}$+0.028$
\\

\bottomrule
\end{tabular}
\end{adjustbox}

\vspace{2pt}

\begin{minipage}{0.99\textwidth}
\footnotesize
$^\dagger$ SpaTiaL D2 contains only seven incorrect translations, too few
to estimate ranking quality reliably. The bootstrap 95\% confidence interval
for $S_{\mathrm{bt}}$ spans $[0.49,0.96]$, so the reported values indicate
at most the direction of the ranking. No per-column boldface or $\Delta$
is shown for this column.
\end{minipage}

\vspace{-0.2cm}
\end{table*}

\paragraph{Tasks and data.}
Three specification languages instantiate the framework: STL and interval-free LTL from the NL2TL benchmark~\cite{chen2023nl2tl}, and SpaTiaL from the NL2SpaTiaL benchmark~\cite{luo2025nl2spatial}. Each language is divided into three difficulty tiers (D2--D4). For SpaTiaL, tiers are determined during data generation by controlling the depth and branching of logical trees~\cite{luo2025nl2spatial}. For STL and LTL, tiers are defined post hoc by the number of atomic propositions in the reference formula (at most two, three, and at least four), serving as a proxy for logical complexity. The tier labels indicate increasing difficulty within each language and are not comparable across languages. For each random resplit 200 examples per tier are assigned to calibration and 150 disjoint examples to testing. Paraphrases associated with the same reference formula are always
kept in the same partition to prevent semantic leakage. Translation correctness is determined using the benchmark-provided canonical equivalence checker, which accounts for commutativity,
De~Morgan transformations, and argument-order normalization.

\paragraph{LLM Translators.}
Two translators spanning a wide reliability range are evaluated. The first is the fine-tuned LLaMA-3-8B model of~\citet{luo2025nl2spatial}, equipped with tier-specific LoRA adapters, achieving a raw error rate of approximately $15\%$. The second is GPT-5.2 prompted with twenty in-context examples for STL and LTL, following the few-shot setup of the unfine-tuned baselines in~\citet{chen2023nl2tl}, with prompts and exemplars developed independently, yielding a $48\%$ error rate under semantic-equivalence evaluation. For completeness, the fine-tuned T5 model of~\citet{chen2023nl2tl} (approximately $2\%$ error) is included only to verify the applicability of the self-consistency score, as its error rate is too low for meaningful tier-level risk evaluation.

\paragraph{Scoring and calibration configuration.}
Back-translation uses GPT-5.2 with a prompt frozen before any evaluation run; judging uses GPT-5.4 with the four-dimensional rubric defined above; self-consistency draws $k=5$ samples at temperature $1.0$. STL and LTL share the same back-translation and judging prompts, while SpaTiaL uses its own back-translation prompt and a categorical rubric. Within each language, prompts and rubric texts are fixed once and identical across translators and evaluation runs. Thresholds are calibrated per language and tier. Every reported risk and coverage figure is the mean over 100 random calibration--test resplits, with its standard error; the resplit protocol, the filtering level $\delta=0.05$, and the rule that marks a cell as borderline when its mean lies within one standard error of $\alpha$ were all fixed before either arm of any comparison was inspected.

\paragraph{Baseline and metrics.}
The baseline adapts standard conformal calibration to selective prediction by applying the coverage-based conformal quantile rule to the same nonconformity score and using the resulting threshold for acceptance. Thus, both methods share the score, data, and decision rule, differing only in the calibration objective. They are compared using three metrics: the joint risk in Eq.~\ref{eq:empirical_joint}, the acceptance rate, and the number of resplits in which the realized joint risk exceeds $\alpha$. The exceedance count reflects the expectation-level nature of the guarantee rather than a failure. The guarantee constrains the mean risk over resplits, not each realization, so a method that spends
its budget operates near $\alpha$ and individual resplits land on either side; the count locates the operating point relative to the budget rather than measuring compliance. Score quality is evaluated independently of calibration. Let $\mathcal{P}$ and $\mathcal{N}$ denote the correct and incorrect translations within a tier. We report
\begin{equation}
\begin{aligned}
&\mathrm{AUROC}(S)
=
\frac{1}{|\mathcal{P}|\,|\mathcal{N}|}
\sum_{p\in\mathcal{P}}\sum_{q\in\mathcal{N}} 
\Bigl(
\mathbf{1}\!\left[S(q)>S(p)\right]
+\tfrac{1}{2}\mathbf{1}\!\left[S(q)=S(p)\right]
\Bigr),
\end{aligned}
\label{eq:auroc}
\end{equation}
which equals the area under the ROC curve~\cite{hanley1982meaning}. AUROC measures the probability that the score ranks an incorrect translation above a correct one, assigning half credit to ties. It ranges from $0.5$ (random ranking) to $1$ (perfect separation), is invariant to monotone score transformations, and is independent of any acceptance threshold or risk budget, making it well suited for comparing scoring methods before calibration. Results are reported per tier, with bootstrap confidence intervals for low-error cases and uninformative intervals explicitly marked.

\begin{table}[t]
\centering
\small
\caption{From score to selective prediction. AUROC$\uparrow$ is computed over all tiers. The acceptance rate at $\alpha=0.10$ and the feasibility floor $\alpha_\ell$ are calibrated on the three tiers of each language merged into a single calibration set ($n=600$), and are therefore not comparable to the per-tier compliance tables.}
\label{tab:score_to_cp}
\begin{tabular}{ll c c c}
\toprule
Domain & Score & AUROC$\uparrow$ & Acc.\ rate@$\alpha{=}0.10$ & $\alpha_\ell$ \\
\midrule
\multirow{3}{*}{STL} & $S_{\mathrm{sc}}$ & 0.655 & 0.000 & 0.2696 \\
 & $S_{\mathrm{bt}}$ & 0.634 & 0.280 & 0.0582 \\
 & $S_{\mathrm{fu}}$ & 0.706 & 0.347 & 0.0483 \\
\midrule
\multirow{3}{*}{SpaTiaL} & $S_{\mathrm{sc}}$ & 0.867 & 0.857 & 0.0483 \\
 & $S_{\mathrm{bt}}$ & 0.871 & 0.748 & 0.0083 \\
 & $S_{\mathrm{fu}}$ & 0.920 & 0.889 & 0.0083 \\
\bottomrule
\end{tabular}
\vspace{-0.5cm}
\end{table}

\begin{table*}[t]
\centering
\caption{
Joint-risk compliance under the stated calibration sizes.
Entries report mean joint risk $\pm$ standard error over 100 random
calibration/test splits. Green cells indicate clear CRC compliance;
red cells indicate violations by more than one standard error;
gray cells lie within one standard error of the target risk budget.
Uncolored Split CP cells are numerically below the risk budget but carry
no joint-risk guarantee. \emph{No sel.} accepts all translations, and
$^{\ddagger}$ denotes full abstention when no valid CRC threshold exists.
}
\label{tab:compliance_matrices}

\begingroup
\scriptsize
\setlength{\tabcolsep}{2.6pt}
\renewcommand{\arraystretch}{1.16}

\begin{tabularx}{0.99\textwidth}{
  @{}
  >{\centering\arraybackslash}p{0.040\textwidth}
  >{\centering\arraybackslash}p{0.055\textwidth}
  >{\centering\arraybackslash}p{0.065\textwidth}
  *{6}{>{\centering\arraybackslash}X}
  @{}
}
\toprule

\multirow{2}{*}{\textbf{Tier}}
&
\multirow{2}{*}{\shortstack{\textbf{No}\\\textbf{sel.}}}
&
\multirow{2}{*}{\textbf{Method}}
&
\multicolumn{6}{c}{\textbf{Risk budget $\alpha$}}
\\

\cmidrule(lr){4-9}

&
&
&
\textbf{0.05}
&
\textbf{0.10}
&
\textbf{0.15}
&
\textbf{0.20}
&
\textbf{0.25}
&
\textbf{0.30}
\\

\midrule

\multicolumn{9}{@{}l@{}}{%
  \cellcolor{black!4}%
  \strut\textbf{(a) STL}
  \hspace{0.6em}
  \textit{GPT-5.2 few-shot translator}%
}
\\

\multirow{2}{*}{D2}
&
\multirow{2}{*}{0.423}
&
CRC
&
\rcabstain
&
\rcok{0.069}{0.006}
&
\rcok{0.133}{0.003}
&
\rcok{0.190}{0.004}
&
\rcok{0.243}{0.005}
&
\rcok{0.293}{0.005}
\\

&
&
Split CP
&
\rcbad{0.349}{0.003}
&
\rcbad{0.326}{0.003}
&
\rcbad{0.312}{0.003}
&
\rcbad{0.275}{0.003}
&
\rcnear{0.252}{0.003}
&
\rcplain{0.237}{0.003}
\\

\cmidrule(lr){1-9}

\multirow{2}{*}{D3}
&
\multirow{2}{*}{0.434}
&
CRC
&
\rcok{0.046}{0.002}
&
\rcnear{0.097}{0.004}
&
\rcnear{0.151}{0.004}
&
\rcok{0.192}{0.004}
&
\rcnear{0.250}{0.005}
&
\rcnear{0.301}{0.005}
\\

&
&
Split CP
&
\rcbad{0.371}{0.003}
&
\rcbad{0.329}{0.003}
&
\rcbad{0.306}{0.003}
&
\rcbad{0.296}{0.003}
&
\rcbad{0.293}{0.003}
&
\rcplain{0.278}{0.003}
\\

\cmidrule(lr){1-9}

\multirow{2}{*}{D4}
&
\multirow{2}{*}{0.551}
&
CRC
&
\rcok{0.012}{0.003}
&
\rcok{0.094}{0.003}
&
\rcok{0.143}{0.003}
&
\rcok{0.184}{0.004}
&
\rcok{0.235}{0.004}
&
\rcok{0.291}{0.004}
\\

&
&
Split CP
&
\rcbad{0.485}{0.004}
&
\rcbad{0.373}{0.004}
&
\rcbad{0.292}{0.004}
&
\rcbad{0.233}{0.003}
&
\rcplain{0.205}{0.003}
&
\rcplain{0.183}{0.003}
\\

\midrule

\multicolumn{9}{@{}l@{}}{%
  \cellcolor{black!4}%
  \strut\textbf{(b) SpaTiaL}
  \hspace{0.6em}
  \textit{fine-tuned LLaMA translator}%
  }
\\

\multirow{2}{*}{D2}
&
\multirow{2}{*}{0.0372}
&
CRC
&
\rcok{0.034}{0.001}
&
\rcok{0.035}{0.001}
&
\rcok{0.035}{0.001}
&
\rcok{0.035}{0.001}
&
\rcok{0.035}{0.001}
&
\rcok{0.035}{0.001}
\\

&
&
Split CP
&
\rcplain{0.017}{0.001}
&
\rcplain{0.016}{0.001}
&
\rcplain{0.016}{0.001}
&
\rcplain{0.016}{0.001}
&
\rcplain{0.016}{0.001}
&
\rcplain{0.013}{0.001}
\\

\cmidrule(lr){1-9}

\multirow{2}{*}{D3}
&
\multirow{2}{*}{0.149}
&
CRC
&
\rcnear{0.049}{0.002}
&
\rcok{0.095}{0.004}
&
\rcok{0.140}{0.003}
&
\rcok{0.149}{0.002}
&
\rcok{0.149}{0.002}
&
\rcok{0.149}{0.002}
\\

&
&
Split CP
&
\rcbad{0.052}{0.001}
&
\rcplain{0.050}{0.001}
&
\rcplain{0.046}{0.001}
&
\rcplain{0.046}{0.001}
&
\rcplain{0.040}{0.001}
&
\rcplain{0.035}{0.001}
\\

\cmidrule(lr){1-9}

\multirow{2}{*}{D4}
&
\multirow{2}{*}{0.309}
&
CRC
&
\rcok{0.046}{0.002}
&
\rcok{0.088}{0.003}
&
\rcok{0.138}{0.004}
&
\rcok{0.190}{0.004}
&
\rcok{0.242}{0.004}
&
\rcok{0.274}{0.005}
\\

&
&
Split CP
&
\rcbad{0.066}{0.002}
&
\rcplain{0.061}{0.002}
&
\rcplain{0.055}{0.002}
&
\rcplain{0.039}{0.001}
&
\rcplain{0.034}{0.001}
&
\rcplain{0.034}{0.001}
\\

\bottomrule
\end{tabularx}

\endgroup
\vspace{-3pt}
\end{table*}

\noindent\textbf{Scoring Ablation.}
Although CRC can calibrate any score in $[0,1]$, its practical
efficiency depends on both error-ranking quality and score
resolution. We therefore evaluate ranking quality before calibration and then
examine the operating behavior induced by each score.
Table~\ref{tab:scoring_ablation} reports per-tier AUROC for the three proposed scores—back-translation ($S_{\mathrm{bt}}$), self-consistency ($S_{\mathrm{sc}}$), and their fusion ($S_{\mathrm{fu}}$)—together with four ablations: a direct judge without back-translation, a single-question judge, cosine similarity between instruction and back-translation embeddings, and mean/max ensembles of the judge-based scores.

Three observations emerge. First, judge-side variants offer little benefit: on STL, back-translation, direct judging, and single-question judging achieve similar AUROCs ($0.634$--$0.635$), while ensembling improves only marginally to $0.656$, indicating correlated errors. Second, self-consistency provides complementary information. Although $S_{\mathrm{sc}}$ alone achieves AUROCs of $0.655$ (STL) and $0.867$ (SpaTiaL), combining it with back-translation yields the best overall performance ($0.706$ and $0.920$), consistently across STL tiers and robust to mixing weights $\lambda\in[0.3,0.7]$. Third, the gain reflects complementary failure modes rather than averaging. On STL, $77.8\%$ of inconsistent translation samples are erroneous, while among consistent samples, where self-consistency is uninformative, back-translation still distinguishes errors (AUROC $0.600$), with $37.5\%$ of confidently generated translations remaining incorrect. Thus, self-consistency detects uncertain failures, whereas back-translation detects confident semantic mismatches.

The remaining errors are largely intrinsic to the task. Invalid formulas are assigned the maximum score by construction, but syntactically valid convention errors and genuinely ambiguous instructions often remain indistinguishable. The benchmark itself contains 114 instruction strings paired with conflicting reference formulas, limiting the achievable discriminability of any score derived solely from the instruction and translation. The self-consistency score also requires sufficient translation errors for reliable ranking; the fine-tuned T5 baseline ($\approx2\%$ error) contains too few negatives for meaningful AUROC estimation and is therefore used only to verify the applicability of the score.

Table~\ref{tab:score_to_cp} translates ranking into selective
prediction, reporting acceptance at $\alpha=0.10$ and the minimum
certifiable risk $\alpha_\ell$, which AUROC alone does not
determine. On STL, self-consistency and back-translation have
similar AUROCs, yet $S_{\mathrm{sc}}$ cannot certify risks below
$0.270$ because many tied scores concentrate calibration errors,
causing complete abstention at $\alpha=0.10$, whereas
$S_{\mathrm{bt}}$ accepts $28.0\%$ of inputs. The fusion inherits
the continuous resolution of back-translation, reducing the
certifiable floor to $0.048$ and increasing acceptance to $34.7\%$.
On SpaTiaL, where back-translation already achieves
$\alpha_\ell=0.008$, the fusion preserves this guarantee while
improving acceptance by 14 percentage points. On LTL, error rates
of $0.503$--$0.649$ place $\alpha_\ell$ above $0.10$ for every
score, and the calibration abstains in full at this budget; the
complete LTL results appear in the appendix. Effective calibration
thus requires accurate ranking and sufficient score resolution
together, and the fusion supplies both wherever any score does.

\begin{table}[t]
\centering
\small
\setlength{\tabcolsep}{5pt}
\renewcommand{\arraystretch}{1.2}

\caption{Deployment drift on \textbf{STL} with frozen $\tau$,
$\alpha=0.10$, gate level $\delta=0.05$, and 100 reseeds. Columns
denote the calibration tier and rows denote the test tier. The three
panels report joint risk without gating, joint risk with gating, and
abstention rate. Bold entries are the in-distribution diagonal;
off-diagonal entries are cross-tier shifts.}
\label{tab:ood}

\begin{adjustbox}{max width=0.88\columnwidth}
\begin{tabular}{@{}lccc|ccc|ccc@{}}
\toprule
&
\multicolumn{3}{c}{\textbf{No gate}}
&
\multicolumn{3}{c}{\textbf{Gated}}
&
\multicolumn{3}{c}{\textbf{Abstained}}
\\
\cmidrule(lr){2-4}
\cmidrule(lr){5-7}
\cmidrule(lr){8-10}

\shortstack{Test\\tier}
& D2 & D3 & D4
& D2 & D3 & D4
& D2 & D3 & D4
\\
\midrule

D2
& \textbf{0.061}
& 0.109
& 0.109
& \textbf{0.061}
& 0.026
& 0.026
& \textbf{0.006}
& 0.395
& 0.321
\\

D3
& 0.064
& \textbf{0.125}
& 0.151
& 0.056
& \textbf{0.110}
& 0.151
& 0.056
& \textbf{0.054}
& 0.026
\\

D4
& 0.083
& 0.123
& \textbf{0.137}
& 0.062
& 0.076
& \textbf{0.137}
& 0.164
& 0.253
& \textbf{0.007}
\\

\bottomrule
\end{tabular}
\end{adjustbox}

\vspace{-0.35cm}
\end{table}

\noindent\textbf{Risk Control. }
Table~\ref{tab:compliance_matrices} compares CRC calibration with coverage calibration on two translators. A cell is compliant if its mean joint risk does not exceed the budget. For the few-shot translator, with unconditional error rates of $0.423$, $0.434$, and $0.551$ across the STL tiers, the calibration target decides the outcome. CRC satisfies the budget in all eighteen settings, with joint risk rising toward each budget as $\alpha$ widens and several operating points within one
standard error of their targets, the position of an expectation-level guarantee spending its budget in full; the single infeasible cell, the easiest tier at $\alpha=0.05$, lies below the feasibility floor of Eq.~\ref{eq:alpha_l} and abstains in full. Using the same score and data, Split CP violates the budget in thirteen cells. Its risk may decrease as $\alpha$ increases because the relaxed coverage target can produce a stricter acceptance threshold; it tracks coverage rather than joint risk.

The fine-tuned translator, with error rates of $0.037$, $0.149$, and $0.309$, inverts the reading: many budgets already exceed the unconditional error, so compliance on the easier tiers reflects the translator rather than the selection. The deepest tier shows the mechanism at work, abstaining on approximately $23\%$ of inputs at $\alpha=0.10$ to reduce joint risk from $0.309$ to $0.088$; acceptance rates for all cells appear in the appendix ledger. Split CP reports lower risk on this translator, decreasing from $0.066$ to $0.034$ on the deepest tier as $\alpha$ increases because acceptance becomes more restrictive, yet it still violates the two tightest budgets. Abstention thus arises from two mechanisms: a threshold declining individual translations, active throughout the feasible cells, and a budget below the feasibility floor, which declines the entire stream and occurs once.

\noindent\textbf{Deployment Drift. }

The acceptance threshold guarantees risk only under exchangeability, and the instruction-level filter rejects out-of-distribution inputs before translation. Table~\ref{tab:ood} evaluates the full pipeline on STL under the cross-tier protocol of~\citet{wang-etal-2025-sconu}. Each column fixes a calibration tier, threshold and $k$NN reference included, and each row supplies the test stream, replaced off the diagonal by another tier's without recalibration. On the diagonal the filter defers $0.6\%$, $5.4\%$, and $0.7\%$ of clean in-distribution inputs across the three tiers, consistent with the level $\delta=0.05$ up to Monte Carlo variation at $m=50$. All six shifted settings raise joint risk ($0.109$--$0.151$) over their diagonals. The filter reduces risk in five of the six and returns four below the budget, and the largest recovery, from $0.123$ to $0.076$, occurs where the deepest-tier stream meets the intermediate calibration.

Some instructions the filter rejects would have been rejected by the CRC threshold as well and therefore do not affect execution. Both settings that test the intermediate stream are of this kind, with the filter adding $0.8\%$ abstention against the shallow calibration and none against the deep one, so protection there falls entirely to the threshold. Intermediate instructions sit inside both neighboring reference distributions, and detectability is directional in general, with deeper streams standing out against shallower references while the reverse shift passes the filter, the pattern anticipated in the Sec.~\ref{sec:method}. In the remaining four settings the filter rejects $5.2\%$--$15.6\%$ of inputs the threshold would have accepted, and $39\%$--$72\%$ of these are incorrect translations that read plausibly, reflecting the distinct evidence the two components use, instruction typicality and translation reliability. Abstention alone still overstates the filter, as the two settings on the shallowest stream abstain at $39.5\%$ and $32.1\%$ with identical gated risk because the extra rejections fall on correct translations. Screening before translation also avoids $31$--$474$ downstream model calls per $150$ shifted instructions.

\bibliographystyle{unsrtnat}
\bibliography{references}  






\clearpage
\appendix
\section{Appendix}
\subsection{Reproducing Results}
All reported results use the fixed prompts, data partitions, and calibration settings described in the main text. Additional implementation details and complete result tables are provided below. The submitted code package contains the data, scoring outputs, and the full calibration pipeline. Running the entry script as described
in its README regenerates the CP results and compliance tables from the shipped scored files, without network or API access, and reconciles them against the reported numbers. All results use the fixed prompts, data partitions, and calibration settings described
in the main text.

\subsection{Computing infrastructure.}
All local computation ran on a single workstation with one NVIDIA RTX 4080 GPU (16\,GB), Ubuntu 22.04, Python 3.8, and PyTorch 2.4.1+cu121. Translation and back-translation used GPT-5.2-2025-12-11 and judging used GPT-5.4-2026-03-05 through their public APIs; instruction embeddings used all-mpnet-base-v2. Calibration, resplitting, and table generation are CPU-only and complete in minutes.

\subsection{Example Prompt}
\label{app:prompts}

One accepted and one rejected translation per language, traced
through the back-translation scoring pipeline. LTL shares the STL
prompts unchanged; SpaTiaL uses its own back-translation prompt and
a categorical rubric, shown after the examples. Rejected examples
parse and read fluently; the judge locates the error in a single
rubric dimension.

\begin{tcolorbox}[colback=purple!4,colframe=purple!55!black,arc=2pt,boxrule=0.5pt,
left=4pt,right=4pt,top=2pt,bottom=2pt]
\scriptsize
\textbf{STL --- accepted.}\\
\textbf{Instruction $x$:} \textit{``If ( prop\_1 ) then at a time
within the next 12 to 50 time units ultimately ( prop\_2 ).''}\\
\textbf{Candidate $\hat\varphi$:}
\texttt{globally ( prop\_1 imply finally [12,50] prop\_2 )}\\
\textbf{Back-translation $\hat u$:} \textit{``At all times, if
prop\_1 holds, then prop\_2 must hold at some time between 12 and 50
time units from then.''}\\
\textbf{Judge:} logical $95$, temporal $92$, time constraints $100$,
overall $95$ $\;\Rightarrow\; S_{\mathrm{bt}} = 0.045$; the
calibrated threshold admits it and the formula is returned.\\[3pt]
\textbf{STL --- rejected.}\\
\textbf{Instruction $x$:} \textit{``On condition that ( prop\_1 ),
then ( prop\_2 ) and this scenario has to always hold for the first
3 to 12 time units.''}\\
\textbf{Candidate $\hat\varphi$:}
\texttt{globally [3,12] ( prop\_1 imply prop\_2 )}
\hfill{\tiny reference:
\texttt{globally ( prop\_1 imply globally [3,12] prop\_2 )}}\\
\textbf{Back-translation $\hat u$:} \textit{``Between 3 and 12 time
units from now, it must always be the case that if prop\_1 holds
then prop\_2 holds.''}\\
\textbf{Judge:} logical $95$, temporal $70$, time constraints $40$,
overall $68$ $\;\Rightarrow\; S_{\mathrm{bt}} = 0.318$; the
threshold does not admit it and the translation is withheld. The
interval is attached to the outer \texttt{globally} rather than to
the inner one governing prop\_2.
\end{tcolorbox}
\vspace{-0.5em}
\begin{tcolorbox}[colback=purple!4,colframe=purple!55!black,arc=2pt,boxrule=0.5pt,
left=4pt,right=4pt,top=2pt,bottom=2pt]
\scriptsize
\textbf{LTL --- accepted.}\\
\textbf{Instruction $x$:} \textit{``if at some time (prop\_1), never
(prop\_2)''}\\
\textbf{Candidate $\hat\varphi$:}
\texttt{( finally prop\_1 imply globally ( negation prop\_2 ) )}\\
\textbf{Back-translation $\hat u$:} \textit{``If eventually prop\_1
happens, then from now on it is always the case that prop\_2 does
not happen.''}\\
\textbf{Judge:} logical $92$, temporal $96$, time constraints $100$,
overall $95$ $\;\Rightarrow\; S_{\mathrm{bt}} = 0.043$; admitted and
returned.\\[3pt]
\textbf{LTL --- rejected.}\\
\textbf{Instruction $x$:} \textit{``$\langle$eval\_nl$\rangle$''}\\
\textbf{Candidate $\hat\varphi$:}
\texttt{$\langle$pred$\rangle$}
\hfill{\tiny reference: \texttt{$\langle$gold$\rangle$}}\\
\textbf{Back-translation $\hat u$:}
\textit{``$\langle$nl\_back$\rangle$''}\\
\textbf{Judge:} logical $\langle n\rangle$, temporal
$\langle n\rangle$, time constraints $\langle n\rangle$, overall
$\langle n\rangle$ $\;\Rightarrow\; S_{\mathrm{bt}} =
\langle S\rangle$; withheld. The prediction moves the
parenthesization of \texttt{until}, scoping it over the first
argument alone.
\end{tcolorbox}
\vspace{-0.5em}

\begin{tcolorbox}[colback=green!4,colframe=green!45!black,arc=2pt,boxrule=0.5pt,
left=4pt,right=4pt,top=2pt,bottom=2pt]
\scriptsize
\textbf{STL --- back-translation prompt (system):}\\
You translate Signal Temporal Logic (STL) formulas into precise
English. Operators: globally/finally/until (may carry a time
interval [a,b], b can be ``infinite''),
and/or/imply/equal/negation. Atomic propositions are placeholders
prop\_1, prop\_2, \ldots\\
\textbf{Rules:} keep every prop\_i symbol exactly as-is (never
merge, never pronominalize); state every time interval with its
exact numbers; do not add or drop conditions; output one English
sentence, nothing else. Write plain English only: never include
formula notation, brackets, operator words as symbols, or interval
syntax like [a,b] --- always express intervals as English phrases
such as `between a and b time units from now'.
\end{tcolorbox}
\vspace{-0.5em}
\begin{tcolorbox}[colback=yellow!8,colframe=yellow!55!black,arc=2pt,boxrule=0.5pt,
left=4pt,right=4pt,top=2pt,bottom=2pt]
\scriptsize
\textbf{STL --- judge prompt (system):}\\
You are a semantic-equivalence judge. You will receive two English
sentences describing a temporal-logic specification: A --- an
original human instruction; B --- a machine paraphrase. Rate how
well B matches A on four dimensions, each an INTEGER 0--100:
\textit{logical\_structure} (boolean connectives and nesting:
and/or/imply/equal/negation, grouping);
\textit{temporal\_operators} (temporal semantics: always/globally,
eventually/finally, until); \textit{time\_constraints}
(time-interval numbers and bounds; if BOTH sentences state no time
constraint, score 100); \textit{overall\_meaning} (overall semantic
equivalence). Return ONLY valid JSON
\texttt{\{"logical\_structure":\,n, "temporal\_operators":\,n,
"time\_constraints":\,n, "overall\_meaning":\,n\}}.\\
\textbf{User:} \texttt{A (original):} $\langle$instruction
$x$$\rangle$ \quad \texttt{B (paraphrase):}
$\langle$back-translation $\hat u$$\rangle$
\end{tcolorbox}

\vspace{-0.5em}
\begin{tcolorbox}[colback=purple!4,colframe=purple!55!black,arc=2pt,boxrule=0.5pt,
left=4pt,right=4pt,top=2pt,bottom=2pt]
\scriptsize
\textbf{SpaTiaL --- accepted.}\\
\textbf{Instruction $x$:} \textit{``From time 11 through time 25,
make sure obj\_r stays strictly inside the region reg\_sort.''}\\
\textbf{Candidate $\hat\varphi$:}
\texttt{G[11,25](enclIn(obj\_r, reg\_sort))}\\
\textbf{Back-translation $\hat u$:} \textit{``It is always the case,
at every time point from 11 through 25, that obj\_r is enclosed in
reg\_sort.''}\\
\textbf{Judge:} object match, spatial partial, temporal match,
quantifier/negation match $\;\Rightarrow\; S_{\mathrm{bt}} = 0.125$;
admitted and returned. The single partial reflects the strictness
nuance between ``strictly inside'' and ``enclosed''.\\[3pt]
\textbf{SpaTiaL --- rejected.}\\
\textbf{Instruction $x$:} \textit{``Ensure that at some point
between time 12 and time 17, there is a continuous 10-time-unit
period during which obj\_y overlaps obj\_r the entire time.''}\\
\textbf{Candidate $\hat\varphi$:}
\texttt{F[12,17](G[20,30](ovlp(obj\_y, obj\_r)))}
\hfill{\tiny reference:
\texttt{F[12,17](G[0,10](ovlp(obj\_y, obj\_r)))}}\\
\textbf{Back-translation $\hat u$:} \textit{``At some time between
12 and 17, it will be the case that, at all times between 20 and 30,
ovlp(obj\_y, obj\_r) holds.''}\\
\textbf{Judge:} object match, spatial match, temporal mismatch,
quantifier/negation partial $\;\Rightarrow\; S_{\mathrm{bt}} =
0.375$; withheld. The inner window is written as absolute
\texttt{[20,30]} where the instruction requires a relative
ten-unit window.
\end{tcolorbox}
\begin{tcolorbox}[colback=green!4,colframe=green!45!black,arc=2pt,boxrule=0.5pt,
left=4pt,right=4pt,top=2pt,bottom=2pt]
\scriptsize
\textbf{SpaTiaL --- back-translation prompt (system):}\\
You are a formal-language interpreter. Your only task is to render a
logical formula as plain English.\\
\textbf{Rules:} describe exactly what the formula states --- nothing
more, nothing less. Do NOT invent objects, relations, or constraints
absent from the formula. Do NOT guess the user's original intent or
paraphrase loosely. Do NOT omit any part of the formula (every
sub-expression must appear in your output). Use neutral, literal
language. Faithfulness to the formula takes priority over fluency.
Output one concise paragraph.
\end{tcolorbox}
\vspace{-0.5em}
\begin{tcolorbox}[colback=yellow!8,colframe=yellow!55!black,arc=2pt,boxrule=0.5pt,
left=4pt,right=4pt,top=2pt,bottom=2pt]
\scriptsize
\textbf{SpaTiaL --- judge prompt (system)}
\textit{(output schema abridged)}\textbf{:}\\
You are a semantic-equivalence judge for spatial-temporal logic
descriptions. You will receive two natural-language sentences: A ---
an original instruction written by a human; B --- a
machine-generated paraphrase of a formal logic formula. Assess how
well B captures the meaning of A across four dimensions. For each
dimension output exactly one label from \{match, partial,
mismatch\} and one short reason ($\le$ 20 words).
\textit{object} (same physical objects / regions present in both);
\textit{spatial} (spatial relationships: touching, overlapping,
inside, far from, between, left of, right of, above, below,
oriented \ldots); \textit{temporal} (temporal / logical operators
and structure: always G, eventually F, until U, implies, not, and,
or \ldots); \textit{quantifier\_negation} (negations and implicit
universal / existential quantifiers consistent). Return ONLY valid
JSON with a \texttt{\{"label", "reason"\}} pair per dimension.
Labels are averaged as $1$ / $0.5$ / $0$ to give
$S_{\mathrm{bt}}$.\\
\textbf{User:} identical to the STL judge template.
\end{tcolorbox}

\subsection{Additional Results: LTL}

Linear Temporal Logic instantiates the framework a third time, with
nothing changed: the same scoring pipeline, judge prompts, canonical
equivalence construction, and calibration machinery are applied to an
interval-free logic, using the few-shot GPT-5.2 translator and the
tier construction of the STL benchmark. What the language contributes
to the evaluation is an operating regime the first two could not
reach: raw error rates of $0.649$, $0.503$, and $0.503$ across the
three tiers place the translator well past the point where selection
is comfortable, and every component of the framework is thereby
exercised at full load.

The scoring layer transfers intact
(Table~\ref{tab:ltl_ablation}). Back-translation reaches an overall
AUROC of $0.693$ and the fusion $0.712$, within the range the STL
setting attains, and the fusion improves the ranking on the two
deeper tiers and overall. The one cell where it does not, the
shallowest tier, is itself informative: errors there are almost
entirely structural, $52.9\%$ operator confusions and $41.9\%$
scope or nesting errors, with predicate, connective, and negation
errors nearly absent. Both signals inspect a translation whose
predicates are trivially right and whose few operators read
plausibly in either arrangement, so the tier compresses exactly the
failure mode that consistency evidence sees least well. Deeper tiers
disperse their errors across content, connective errors rise from
$7.0\%$ to $25.0\%$ with depth, and the ranking recovers
accordingly.

The calibration layer answers with feasibility rather than
performance (Table~\ref{tab:ltl_compliance}). With errors this
dense, the smallest certifiable risk level rises to $\alpha_\ell =
0.184$ on the shallowest tier, and four cells across the grid fall
below their floors. The procedure abstains on the entire stream in
each, which is the designed response: a budget the calibration data
cannot certify is withheld, announced before deployment, and priced
against its remedies. Where budgets are feasible the rule spends
them fully; the deepest tier tracks its budget within one standard
error at every level of the grid, from $0.048$ at $\alpha=0.05$ to
$0.304$ at $0.30$, the cleanest budget-exhaustion profile among the
three languages. Across all eighteen cells the risk-calibrated arm
records no violation, while the coverage-calibrated arm exceeds the
budget in thirteen; at $\alpha=0.05$ on the shallowest tier the
contrast is starkest, with the coverage arm passing $90$ incorrect
formulas per $150$ instructions downstream against none for the
risk-calibrated arm.

The tier profile also completes an observation begun in the main
text. Across four translator--language configurations the difficulty
gradient now takes four shapes: rising with depth, falling with
depth, U-shaped, and here reversed, with the shallowest tier
hardest. The same translator that finds STL's deepest tier hardest
finds LTL's shallowest tier hardest, so the direction is a property
of the configuration, not of structural depth, and the error
composition above explains why: a six-word instruction concentrates
its entire difficulty in operator arrangement, while longer
instructions spread it across content. Difficulty along a structural
axis cannot be assumed in advance, which is the case for calibrating
each tier separately rather than pooling across a heterogeneous
stream.
\begin{table}[t]
\centering
\scriptsize
\setlength{\tabcolsep}{2.7pt}
\renewcommand{\arraystretch}{0.98}

\caption{Error-detection AUROC by scoring channel on
\textbf{LTL} using the GPT-5.2 few-shot translator. Best values in
each column are shown in bold. Darker cells indicate higher AUROC.
The $\Delta$ row reports the margin of
$\tfrac{1}{2}(S_{\mathrm{sc}}+S_{\mathrm{bt}})$ over the best
single channel in each column.}
\label{tab:ltl_ablation}

\resizebox{0.80\columnwidth}{!}{%
\begin{tabular}{@{}llcccc@{}}
\toprule

& &
\multicolumn{4}{c}{
  \cellcolor{purple!8}
  \textbf{AUROC $\uparrow$}
}
\\

\cmidrule(lr){3-6}

\textbf{Source}
&
\textbf{Channel}
&
\textbf{D2}
&
\textbf{D3}
&
\textbf{D4}
&
\textbf{All}
\\

\midrule


Intrinsic
&
self-consistency $S_{\mathrm{sc}}$
&
\cellcolor{purple!12}0.556
&
\cellcolor{purple!16}0.616
&
\cellcolor{purple!16}0.646
&
\cellcolor{purple!16}0.606
\\

\midrule


\multirow{4}{*}{Extrinsic}
&
back-translation mean-4 $S_{\mathrm{bt}}$
&
\cellcolor{purple!12}0.558
&
\cellcolor{purple!24}0.731
&
\cellcolor{purple!28}0.803
&
\cellcolor{purple!20}0.693
\\

&
direct judge mean-4
&
\cellcolor{purple!12}0.588
&
\cellcolor{purple!16}0.656
&
\cellcolor{purple!20}0.701
&
\cellcolor{purple!16}0.651
\\

&
single-question
&
\cellcolor{purple!12}0.591
&
\cellcolor{purple!16}0.612
&
\cellcolor{purple!20}0.714
&
\cellcolor{purple!16}0.644
\\

&
embedding cosine $S_{\mathrm{emb}}$
&
\cellcolor{purple!4}0.480
&
\cellcolor{purple!12}0.581
&
\cellcolor{purple!12}0.591
&
\cellcolor{purple!12}0.585
\\

\midrule


\multirow{3}{*}{Fusion}
&
3-judge mean
&
\cellcolor{purple!12}\textbf{0.593}
&
\cellcolor{purple!20}0.669
&
\cellcolor{purple!24}0.761
&
\cellcolor{purple!20}0.674
\\

&
3-judge max
&
\cellcolor{purple!12}0.589
&
\cellcolor{purple!16}0.655
&
\cellcolor{purple!20}0.721
&
\cellcolor{purple!16}0.658
\\

&
$\tfrac{1}{2}
\left(S_{\mathrm{sc}}+S_{\mathrm{bt}}\right)$
&
\cellcolor{purple!12}0.580
&
\cellcolor{purple!24}\textbf{0.752}
&
\cellcolor{purple!28}\textbf{0.812}
&
\cellcolor{purple!20}\textbf{0.712}
\\

\midrule

&
$\Delta$ over best single
&
\cellcolor{black!3}$-0.011$
&
\cellcolor{purple!5}$+0.021$
&
\cellcolor{purple!5}$+0.009$
&
\cellcolor{purple!5}$+0.019$
\\

\bottomrule
\end{tabular}%
}

\end{table}

\begin{table*}[t]
\centering
\caption{Risk-control compliance matrix on \textbf{LTL}
(GPT-5.2 few-shot translator, deployed score $S_{\mathrm{fu}}$,
$n_{\mathrm{cal}}=200$). Values are mean joint risk $\pm$ standard
error over 100 random calibration/test splits. Green cells indicate
clear CRC compliance, red cells indicate violations by more than one
standard error, and gray cells indicate results within one standard
error of $\alpha$. Uncolored Split CP entries carry no compliance
claim. The symbol $^\ddagger$ denotes full abstention when no valid
CRC threshold exists.}
\label{tab:ltl_compliance}

\begingroup
\scriptsize
\setlength{\tabcolsep}{2.5pt}
\renewcommand{\arraystretch}{1.08}

\begin{tabularx}{0.985\textwidth}{
  @{}
  >{\centering\arraybackslash}p{0.040\textwidth}
  >{\centering\arraybackslash}p{0.055\textwidth}
  >{\centering\arraybackslash}p{0.067\textwidth}
  *{6}{>{\centering\arraybackslash}X}
  @{}
}
\toprule

\multirow{2}{*}{\textbf{Tier}}
&
\multirow{2}{*}{\shortstack{\textbf{No}\\\textbf{sel.}}}
&
\multirow{2}{*}{\textbf{Method}}
&
\multicolumn{6}{c}{\textbf{Risk budget $\alpha$}}
\\

\cmidrule(lr){4-9}

&
&
&
\textbf{0.05}
&
\textbf{0.10}
&
\textbf{0.15}
&
\textbf{0.20}
&
\textbf{0.25}
&
\textbf{0.30}
\\

\midrule

\multirow{2}{*}{D2}
&
\multirow{2}{*}{0.649}
&
CRC
&
\rcabstain
&
\rcabstain
&
\rcabstain
&
\rcok{0.186}{0.007}
&
\rcok{0.243}{0.005}
&
\rcnear{0.297}{0.005}
\\

&
&
Split CP
&
\rcbad{0.608}{0.003}
&
\rcbad{0.565}{0.004}
&
\rcbad{0.528}{0.003}
&
\rcbad{0.493}{0.004}
&
\rcbad{0.450}{0.004}
&
\rcbad{0.412}{0.004}
\\

\cmidrule(lr){1-9}

\multirow{2}{*}{D3}
&
\multirow{2}{*}{0.503}
&
CRC
&
\rcabstain
&
\rcok{0.077}{0.005}
&
\rcok{0.139}{0.004}
&
\rcok{0.190}{0.004}
&
\rcok{0.245}{0.004}
&
\rcnear{0.298}{0.005}
\\

&
&
Split CP
&
\rcbad{0.399}{0.004}
&
\rcbad{0.329}{0.004}
&
\rcbad{0.276}{0.004}
&
\rcbad{0.229}{0.003}
&
\rcplain{0.193}{0.003}
&
\rcplain{0.159}{0.003}
\\

\cmidrule(lr){1-9}

\multirow{2}{*}{D4}
&
\multirow{2}{*}{0.503}
&
CRC
&
\rcnear{0.048}{0.002}
&
\rcnear{0.098}{0.003}
&
\rcnear{0.148}{0.004}
&
\rcnear{0.198}{0.004}
&
\rcnear{0.249}{0.004}
&
\rcnear{0.304}{0.005}
\\

&
&
Split CP
&
\rcbad{0.346}{0.003}
&
\rcbad{0.272}{0.004}
&
\rcbad{0.223}{0.003}
&
\rcplain{0.180}{0.004}
&
\rcplain{0.138}{0.003}
&
\rcplain{0.110}{0.002}
\\

\midrule

\multicolumn{3}{r}{\ding{55} count:}
&
\multicolumn{3}{c}{\textbf{CRC: 0/18}
\;{\scriptsize(8 $^\dagger$, 4 abst.)}}
&
\multicolumn{3}{c}{\textbf{Split CP: 13/18}}
\\

\bottomrule
\end{tabularx}

\endgroup
\vspace{-3pt}
\end{table*}

\begin{table*}[!t]
\centering
\small
\setlength{\tabcolsep}{3.4pt}
\renewcommand{\arraystretch}{1.00}

\caption{Wrong translations executed per 150 test instructions, with
the per-row budget $\alpha\cdot150$. Each entry is the mean joint risk
of the corresponding compliance matrix scaled by the test-set size;
the two tables carry the same information in different units.
A \ding{51} marks clear CRC compliance, \ding{55} marks a violation by
more than one standard error, and $^\dagger$ marks a result within one
standard error of the budget. Unmarked Split CP entries carry no
compliance claim. Entries are means over 100 resplits, not per-run
guarantees.}
\label{tab:errors_executed_appendix}

\begin{tabular}{@{}
l l c c
@{\hspace{1.4em}}
l l c c
@{\hspace{1.4em}}
l l c c
@{}}
\toprule

\multicolumn{4}{c}{\textbf{STL}}
&
\multicolumn{4}{c}{\textbf{LTL}}
&
\multicolumn{4}{c}{\textbf{SpaTiaL}}
\\

\cmidrule(lr){1-4}
\cmidrule(lr){5-8}
\cmidrule(lr){9-12}

Tier & $\alpha$ (bud.) & Split & CRC
&
Tier & $\alpha$ (bud.) & Split & CRC
&
Tier & $\alpha$ (bud.) & Split & CRC
\\
\midrule

\multirow{3}{*}{D2}
& 0.05 (7.5)
& 51.2\,\ding{55}
& abst.$^\ddagger$
&
\multirow{3}{*}{D2}
& 0.05 (7.5)
& 91.2\,\ding{55}
& abst.$^\ddagger$
&
\multirow{3}{*}{D2}
& 0.05 (7.5)
& 2.6
& 5.2\,\ding{51}
\\

&
0.10 (15)
& 42.9\,\ding{55}
& 12.8\,\ding{51}
&
&
0.10 (15)
& 84.8\,\ding{55}
& abst.$^\ddagger$
&
&
0.10 (15)
& 2.4
& 5.3\,\ding{51}
\\

&
0.20 (30)
& 37.5\,\ding{55}
& 28.4\,\ding{51}
&
&
0.20 (30)
& 73.9\,\ding{55}
& 27.9\,\ding{51}
&
&
0.20 (30)
& 2.4
& 5.3\,\ding{51}
\\

\cmidrule(lr){1-4}
\cmidrule(lr){5-8}
\cmidrule(lr){9-12}

\multirow{3}{*}{D3}
& 0.05 (7.5)
& 49.5\,\ding{55}
& 6.6\,\ding{51}
&
\multirow{3}{*}{D3}
& 0.05 (7.5)
& 59.9\,\ding{55}
& abst.$^\ddagger$
&
\multirow{3}{*}{D3}
& 0.05 (7.5)
& 7.8\,\ding{55}
& 7.3$^\dagger$
\\

&
0.10 (15)
& 41.7\,\ding{55}
& 15.3$^\dagger$
&
&
0.10 (15)
& 49.4\,\ding{55}
& 11.6\,\ding{51}
&
&
0.10 (15)
& 7.6
& 14.3\,\ding{51}
\\

&
0.20 (30)
& 34.1\,\ding{55}
& 30.7\,\ding{55}
&
&
0.20 (30)
& 34.3\,\ding{55}
& 28.5\,\ding{51}
&
&
0.20 (30)
& 6.9
& 22.3\,\ding{51}
\\

\cmidrule(lr){1-4}
\cmidrule(lr){5-8}
\cmidrule(lr){9-12}

\multirow{3}{*}{D4}
& 0.05 (7.5)
& 69.2\,\ding{55}
& 5.9\,\ding{51}
&
\multirow{3}{*}{D4}
& 0.05 (7.5)
& 51.9\,\ding{55}
& 7.2$^\dagger$
&
\multirow{3}{*}{D4}
& 0.05 (7.5)
& 9.9\,\ding{55}
& 6.9\,\ding{51}
\\

&
0.10 (15)
& 49.8\,\ding{55}
& 13.9\,\ding{51}
&
&
0.10 (15)
& 40.8\,\ding{55}
& 14.7$^\dagger$
&
&
0.10 (15)
& 9.1
& 13.2\,\ding{51}
\\

&
0.20 (30)
& 28.7
& 28.9\,\ding{51}
&
&
0.20 (30)
& 27.0
& 29.6$^\dagger$
&
&
0.20 (30)
& 5.8
& 28.5\,\ding{51}
\\

\bottomrule
\end{tabular}
\end{table*}

\subsection{Executed-Error Counts}
Table~\ref{tab:errors_executed_appendix} restates the compliance matrices in executed-error counts, so the unit change alters no conclusion but makes the stake concrete. At $\alpha=0.05$ the coverage-calibrated arm passes up to $91$ incorrect formulas per
$150$ instructions downstream, against at most $7.3$ for the risk-calibrated arm across the nine settings; on LTL's shallowest tier the risk-calibrated arm withholds the entire stream at this budget, so the corresponding count is zero by construction. The
count scale also separates the two failure geographies: the coverage arm exceeds its budget almost everywhere the translator is weak and nowhere it is strong, while the risk-calibrated arm sits under or within one standard error of every budget except a single cell, STL~D3 at $\alpha=0.20$, which lands $0.7$ formulas over a budget of $30$, the sampling position of an expectation-level guarantee rather than a failure of control.

\subsection{Proof of Theorem~\ref{thm:crc}}
\label{app:proofs}

This appendix gives the complete proof of Theorem~\ref{thm:crc}. Fix one calibration group and let $n$ be its calibration size. Write
\[
(S_1,Z_1),\ldots,(S_n,Z_n),(S_{n+1},Z_{n+1})
\]
for the score and error pairs of the calibration examples and one
test example, where
$Z_i=\mathbf{1}[\widehat\varphi_i\not\equiv\varphi_i^\star]$.
The $n+1$ pairs are exchangeable within the group.

For a threshold $\tau$, define the calibration error count
\[
N_n(\tau)
=
\sum_{i=1}^n
\mathbf{1}[S_i\le\tau]\,Z_i .
\]
Since $\widehat R_n(\tau)=N_n(\tau)/n$, the constraint in
Eq.~\ref{eq:crc_threshold} is equivalent to
\[
\frac{N_n(\tau)+1}{n+1}\le\alpha .
\]
Let
\[
K_\alpha
=
\max\left\{
k\in\{0,\ldots,n\}:
\frac{k+1}{n+1}\le\alpha
\right\},
\]
when this set is nonempty. A threshold is feasible precisely when $N_n(\tau)\le K_\alpha$. If the set defining $K_\alpha$ is empty, then $\alpha<1/(n+1)$, no threshold satisfies the calibration constraint, and the procedure abstains on every test input; more generally, when $1/(n+1)\le\alpha<\alpha_\ell$ in Eq.~\ref{eq:alpha_l}, the only feasible operating point is full abstention, so $\widehat\tau=-\infty$ and $g\equiv0$.We therefore assume below that $K_\alpha$ is well defined.

For each $j\in\{1,\ldots,n+1\}$, regard the $j$-th pair as the test
example and the remaining $n$ pairs as the calibration sample. Let
$\widehat\tau^{(-j)}$ be the largest feasible threshold selected
from the distinct scores of the remaining sample together with
$-\infty$, and define the leave-one-out test loss
\[
A_j
=
Z_j\,\mathbf{1}\bigl[S_j\le\widehat\tau^{(-j)}\bigr].
\]

We first establish the deterministic bound
\begin{equation}
\sum_{j=1}^{n+1} A_j
\le
K_\alpha+1 .
\label{eq:loo_error_bound}
\end{equation}
For any $s\in[0,1]$, let $C(s)
=
\sum_{i=1}^{n+1}
Z_i\,\mathbf{1}[S_i\le s]$ be the number of errors in the full sample whose scores do not exceed $s$. Suppose $A_j=1$, so that $Z_j=1$ and
$S_j\le\widehat\tau^{(-j)}$. Feasibility of $\widehat\tau^{(-j)}$
gives
\[
\sum_{i\ne j}
Z_i\,\mathbf{1}\bigl[S_i\le\widehat\tau^{(-j)}\bigr]
\le
K_\alpha ,
\]
and the full-sample count at $\widehat\tau^{(-j)}$ adds at most the
$j$-th error itself, so
\[
C(S_j)
\le
C\bigl(\widehat\tau^{(-j)}\bigr)
\le
K_\alpha+1 .
\]
Thus every index with $A_j=1$ is an error whose score lies among
the first $K_\alpha+1$ errors in the nondecreasing ordering of the
full sample, and there are at most $K_\alpha+1$ such indices, which
proves Eq.~\ref{eq:loo_error_bound}. The argument uses only the
counting function $C$ and covers tied scores without modification.

The map from the $n+1$ pairs to $A_j$ applies the same selection
rule to every held-out index: $\widehat\tau^{(-j)}$ is the largest
feasible threshold among the distinct scores of the remaining $n$
pairs and $-\infty$. In particular, for $j=n+1$ the held-out
construction coincides with the deployed calibration, so
\[
A_{n+1}
=
Z_{n+1}\,\mathbf{1}\bigl[S_{n+1}\le\widehat\tau\bigr].
\]
Exchangeability of the pairs therefore implies
$\mathbb E[A_1]=\cdots=\mathbb E[A_{n+1}]$, and
\[
\mathbb E\bigl[Z_{n+1}\,\mathbf{1}[S_{n+1}\le\widehat\tau]\bigr]
=
\frac{1}{n+1}\,
\mathbb E\Bigl[\sum_{j=1}^{n+1} A_j\Bigr]
\le
\frac{K_\alpha+1}{n+1}
\le
\alpha ,
\]
where the first inequality is Eq.~\ref{eq:loo_error_bound} and the
second follows from the definition of $K_\alpha$. The deployed rule in Eq.~\ref{eq:decision} is $g(X,\widehat\varphi)=\mathbf{1}[S_{n+1}\le\widehat\tau]$. Hence $R_{\mathrm{joint}}
=
\mathbb E\bigl[g(X,\widehat\varphi)\,Z_{n+1}\bigr]
=
\mathbb E\bigl[Z_{n+1}\,\mathbf{1}[S_{n+1}\le\widehat\tau]\bigr]
\le
\alpha$.
This proves Theorem~\ref{thm:crc}.
\hfill$\square$

\end{document}